\documentclass[11pt,letterpaper]{article}
\usepackage[letterpaper]{geometry}
\usepackage{ifthen}

\newcommand\forarxiv{}

\ifthenelse{\isundefined{\forarxiv}}{
  \usepackage{acl2012}
  \usepackage{times}
}{
  \usepackage{fullpage}
  \setlength{\textheight}{9in}
}

\usepackage{latexsym}
\usepackage{amsmath}
\usepackage{multirow}
\usepackage{url}
\usepackage{tabularx}

\makeatletter
\newcommand{\@BIBLABEL}{\@emptybiblabel}
\newcommand{\@emptybiblabel}[1]{}
\makeatother
\usepackage[hidelinks]{hyperref}

\usepackage{amsmath,amsbsy,amsfonts,amssymb,amsthm}
\usepackage{dsfont,units,psfrag}
\usepackage{mathtools}
\usepackage{bbm}
\usepackage[ruled,vlined]{algorithm2e}
\usepackage[noend]{algorithmic} 
\usepackage{enumitem}
\usepackage{color,cases}
\usepackage{tikz}
\usetikzlibrary{calc, shapes, backgrounds, arrows, fit, positioning}
\usepackage{float}

\def\tl{\tilde}

 \def\0{{\bf 0}}

\def\qed{\hfill\hbox{${\vcenter{\vbox{
    \hrule height 0.4pt\hbox{\vrule width 0.4pt height 6pt
    \kern5pt\vrule width 0.4pt}\hrule height 0.4pt}}}$}}

\newcommand{\bprf}{\begin{myproof}}
\newcommand{\eprf}{\end{myproof}}
\newcommand{\bp}{\begin{psfrags}}
\newcommand{\ep}{\end{psfrags}}
\newcommand{\bl}{\begin{lemma}}
\newcommand{\el}{\end{lemma}}
\newcommand{\bt}{\begin{theorem}}
\newcommand{\et}{\end{theorem}}
\newcommand{\bc}{\begin{center}}
\newcommand{\ec}{\end{center}}
\newcommand{\bi}{\begin{itemize}}
\newcommand{\ei}{\end{itemize}}
\newcommand{\ben}{\begin{enumerate}}
\newcommand{\een}{\end{enumerate}}
\newcommand{\bd}{\begin{definition}}
\newcommand{\ed}{\end{definition}}
\def\beq{\begin{equation}}
\def\eeq{\end{equation}\noindent}
\def\beqn{\begin{eqnarray}}
\def\eeqn{\end{eqnarray} \noindent}
\def\beqnn{  \begin{eqnarray*}}
\def\eeqnn{\end{eqnarray*}  \noindent}
\def\bcase{  \begin{numcases}}
\def\ecase{\end{numcases}   \noindent}
\def\bsbcase{  \begin{subnumcases}}
\def\esbcase{\end{subnumcases}   \noindent}

\newtheorem{theorem}{Theorem}

\newtheorem{lemma}[theorem]{Lemma}

\newtheorem{definition}{Definition}

\newenvironment{myproof}{\noindent{\bf Proof:} \hspace*{1em}}{
    \hspace*{\fill} $\Box$ }
\newenvironment{proof_of}[1]{\noindent {\bf Proof of #1: }}{\hspace*{\fill} $\Box$ }

\newcommand{\matplottc}[1]{               
        \unitlength .45truein
        \begin{center}
        \includegraphics{#1.ps}
        \end{picture}
        \end{center}
}

\def\psfancypar#1#2{\begingroup\def\par{\endgraf\endgroup\lineskiplimit=0pt}
               \setbox2=\hbox{\large\sc #2}
               \newdimen\tmpht \tmpht \ht2 \advance\tmpht by \baselineskip
               \font\hhuge=Times-Bold at \tmpht
               \setbox1=\hbox{{\hhuge #1}}
               \count7=\tmpht \count8=\ht1
               \divide\count8 by 1000 \divide\count7 by \count8
               \tmpht=.001\tmpht\multiply\tmpht by \count7
               \font\hhuge=Times-Bold at \tmpht
               \setbox1=\hbox{{\hhuge #1}}
               \noindent
                \hangindent1.05\wd1
               \hangafter=-2 {\hskip-\hangindent
               \lower1\ht1\hbox{\raise1.0\ht2\copy1}%
                \kern-0\wd1}\copy2\lineskiplimit=-1000pt}

\def\Kout{\setbox1=\hbox{\Huge\bf K}\hbox to
1.05\wd1{\hspace{.05\wd1}
}}
\def\Sout{\setbox1=\hbox{\Huge\bf S}\hbox to 1.05\wd1{\hspace{.05\wd1}
}}

\newcommand\Tstrut{\rule{0pt}{2.4ex}}       

\newcommand\itrs{ITRS}
\newcommand\iet{IET}
\newcommand\corr{\mathit{corr}}

\ifthenelse{\isundefined{\forarxiv}}{
  \newcommand\citet{\newcite}
  \newcommand\plotwidth{0.49\textwidth}

}{
  \usepackage[round]{natbib}
  \renewcommand\cite{\citep}
  \newcommand\plotwidth{0.6\textwidth}

}

\ifthenelse{\isundefined{\forarxiv}}{
  \title{Knowledge Completion for Generics using Guided Tensor Factorization}
  \author{
    Hanie Sedghi\thanks{\ \ This work was done while the author was affiliated with the Allen Institute for Artificial Intelligence.}\\
    Google Brain\\
    Mountain View, CA, U.S.A.\\
    {\tt \small hsedghi@google.com}
    \And
    Ashish Sabharwal\\
    Allen Institute for Artificial Intelligence (AI2)\\
    Seattle, WA, U.S.A.\\
    {\tt \small AshishS@allenai.org}
  }
}{
  \title{Knowledge Completion for Generics\\ using Guided Tensor Factorization}
  \author{
    Hanie Sedghi\thanks{This work was done while the author was affiliated with the Allen Institute for Artificial Intelligence.}\\
    Google Brain\\
    Mountain View, CA, U.S.A.\\
    {\tt \small hsedghi@google.com}
    \and
    Ashish Sabharwal\\
    Allen Institute for Artificial Intelligence (AI2)\\
    Seattle, WA, U.S.A.\\
    {\tt \small AshishS@allenai.org}
  }
}

\date{}

\ifthenelse{\isundefined{\forarxiv}}{
}{
  \pagestyle{plain}
}

\begin{document}

\maketitle


\begin{abstract}

Given a knowledge base or KB containing (noisy) facts about common nouns or generics, such as ``all trees produce oxygen" or ``some animals live in forests", we consider the problem of inferring additional such facts at a precision similar to that of the starting KB. Such KBs capture general knowledge about the world, and are crucial for various applications such as question answering. 
Different from commonly studied named entity KBs such as Freebase, generics KBs involve quantification,
have more complex underlying regularities, tend to be more incomplete, 
and violate the commonly used locally closed world assumption (LCWA).
We show that existing KB completion methods struggle with this new task, and present the first approach that is successful.
Our results demonstrate that external information, such as relation schemas and entity taxonomies, if used appropriately, can be a surprisingly powerful tool in this setting. First, our simple yet effective knowledge guided tensor factorization approach achieves state-of-the-art results on two generics KBs (80\% precise) for science, doubling their size at 74\%-86\% precision. Second, our novel taxonomy guided, submodular, active learning method for collecting annotations about rare entities (e.g., oriole, a bird) is 6x more effective at inferring further new facts about them than multiple active learning baselines.

\end{abstract}

\section{Introduction}

We consider the problem of completing a partial knowledge base (KB) containing facts about generics or common nouns, represented as a third-order tensor of \emph{(source, relation, target)} triples, such as \emph{(butterfly, pollinate, flower)} and \emph{(thermometer, measure, temperature)}. Such facts capture common knowledge that humans have about the world. They are arguably essential for intelligent agents with human-like conversational abilities as well as for specific applications such as question answering. We demonstrate that state-of-the-art KB completion methods perform poorly when faced with generics, while our strategies for incorporating external knowledge as well as obtaining additional annotations for rare entities provide the first successful solution to this challenging new task.

Since generics represent classes of similar individuals, the truth value $y_i$ of a generics triple $x_i = (s,r,t)$ depends on the quantification semantics one associates with $s$ and $t$. Indeed, the semantics of generics statements can be ambiguous, even self-contradictory, due to cultural norms. As \citet{Leslie2008GenericsCA} points out, `ducks lay eggs' is generally considered true while `ducks are female', which is true for a broader set of ducks than the former statement, is generally considered false.

To avoid deep philosophical issues, we fix a particular mathematical semantics that is especially relevant for noisy facts derived automatically from text: associate $s$ with a categorical quantification from $\{all, some, none\}$ and associate $t$ (implicitly) with $some$. For instance, ``all butterflies pollinate (some) flower'' and ``some animals live in (some) forest''. When presenting such triples to humans, they are phrased as: \emph{is it true that all butterflies pollinate some flower?} As a notational shortcut, we treat the quantification of $s$ as the categorical label $y_i$ for the triple $x_i$. E.g., \emph{(butterfly, pollinate, flower)} is labeled $all$ while \emph{(animal, live in, forest)} is labeled $some$. Given a noisy KB of such labeled triples, the task is to infer more triples. 

Tensor factorization and graph based methods have both been found to be very effective for expanding knowledge bases, but have focused on \emph{named entity} KBs such as Freebase~\cite{bollacker2008freebase}
involving relations with clear semantics such as \emph{liveIn} and \emph{isACityIn}, and disambiguated entities such as \emph{Barack Obama} or \emph{Hawaii}. On the other hand, completing KBs that involve facts about generics surfaces new challenges, as evidenced by our empirical results when using existing methods.

It has been observed that Horn clauses often reliably connect predicates in the named-entity setting. For instance, for any person \emph{x}, city \emph{y}, and country \emph{z}, \emph{(x, liveIn, y)} \& \emph{(y, isACityIn, z)} $\Rightarrow$ \emph{(x, liveIn, z)}. With generics, however, clear patterns or reliable first-order logic rules are rare, in part due to each generic representing a collection of individuals that often have similarities with respect to some relations and differences with respect to others. For instance, \emph{(x, liveIn, mountain)} is true for many \emph{cats} and \emph{caribou}, but there is little tangible similarity between the two animals and it is unclear what, if anything, can be carried over from one to the other. On the other hand, if we take two animals that share a `parent' in some taxonomy (e.g., \emph{raindeer} and \emph{deer}), then the likelihood of knowledge transfer increases.


We propose to make use of additional rich background knowledge complementing the information present in the KB itself, such as a \emph{taxonomic hierarchy} of entities (available from sources such as WordNet~\cite{miller1995wordnet}) and the corresponding \emph{entity types} and \emph{relation schema}.  Our key insight is that, if used appropriately, \emph{taxonomic and schema information can be surprisingly effective in making tensor factorization methods vastly more effective for generics} for deriving high precision facts.

Intuitively, for generics, many properties of interest are themselves generic (e.g., living in forests, as opposed to living in a specific forest) and tend to be shared by siblings in a taxonomy (e.g., finch, oriole, and hummingbird). In contrast, siblings of named entities (e.g., various people) often differ substantially in the properties we typically care about and model (e.g., who they are married to, where they live, etc.). Methods that use type information are thus more promising for generics than for classical NLP tasks involving named entities. We propose three ways of using this information and empirically demonstrate the effectiveness of each on two variants of a KB of elementary level science facts~\cite{dalvi2017domaintargeted}.\footnote{We are unaware of other large generics KBs. Our method does not employ rules or choices specific to this dataset and is expected to generalize to other generics KBs, as and when they become available.}

First, we observe that simply imposing \textbf{schema consistency} (Section~\ref{sec:itrs}) on derived facts can significantly boost state-of-the-art methods such as Holographic Embeddings (HolE) \cite{Nickel2016HolographicEO} from nearly no new facts at 80\% precision to over 10,000 new facts, starting with a generics KB of a similar size. Other embedding methods, such as TransE~\cite{bordes2013translating}, RESCAL~\cite{nickel2011three}, and SICTF~\cite{nimishakavi2016relation} (which uses schema information as well), also produced no new facts at 80\% precision.
Graph-based completion methods did not scale to our densely connected tensors.\footnote{\label{footnote:graph-based}On the smaller Animals tensor (to be described later), PRA~\cite{lao2011random} generated very few high-precision facts after 30 hours. SFE~\cite{Gardner2015EfficientAE} was unable to finish training a classifier for any relation after a day, in part due to the high connectivity of generics like \emph{animal}. On the other hand, HolE is trained in a couple of minutes even on the larger Science tensor, and can be made even faster using the method of \citet{Hayashi2017OnTE}.} 

Second, one can further boost performance by transferring knowledge up and down the \textbf{taxonomic hierarchy}, using the quantification semantics of generics (Section~\ref{sec:iet}). We show that expanding the starting tensor this way before applying tensor factorization is complementary and results in a statistically significantly higher precision (86.4\% as opposed to 82\%) over new facts at the same yield.

Finally, we propose a novel \textbf{limited-budget taxonomy guided active learning} method to address the challenge of significant incompleteness in generics KBs, by
quantifying uncertainty via siblings (Section~\ref{sec:AL}).
\citet{dalvi2017domaintargeted} have observed that reliable facts about generics are much harder to derive using information extraction methods than facts about named entities. 
This makes generics KBs vastly incomplete, with no or very little information about certain entities such as caribou or oriole.

Our active learning approach addresses the following question: \emph{Given a new entity\footnote{Unless otherwise stated, we will henceforth use \emph{entity} to refer to a singular common noun that represents a class or group of individuals, such as \emph{animal, hummingbird, forest,} etc.} $\tl{e}$ and a budget $B$, what is a good set $Q$ of $B$ queries about $\tl{e}$ to annotate (via humans) such that expanding the original tensor with $Q$ helps a KB completion method infer many more high precision facts about $\tl{e}$?}

%
We propose to define a correlation based measure of the uncertainty of each unannotated triple (i.e., a potential query) involving $\tl{e}$ based on how frequently the corresponding triple is true for $\tl{e}$'s siblings in the taxonomic hierarchy (Section~\ref{sec:AL1}).
%
We then develop a submodular objective function, and a corresponding greedy $(1 - 1/e)$-approximation, to search for a small subset of triples to annotate that optimally balances diversity with coverage 
(Section~\ref{sec:AL2}).
We demonstrate that annotating this balanced subset makes tensor factorization derive substantially more new and interesting facts
compared to several active learning baselines. For example, with a budget to annotate 100 queries about a new entity oriole, random queries lead to no new true facts at all (via annotation followed by tensor factorization), imposing schema consistency results in 83 new facts, and our proposed method ends up with 483 new facts. This demonstrates that well-designed intelligent queries can be substantially more effective in gathering facts about the new entity.

In summary, this work tackles for the first time the challenging task of knowledge completion for generics, by imposing consistency with external knowledge. Our efficient sibling-guided active learning approach addresses the paucity of facts about certain entities, successfully inferring a substantial number of new facts about them.

\subsection{Related Work}

KB completion approaches fall into two main classes: graph-based methods and those employing low-dimensional embeddings via matrix or tensor factorization. The former use graph traversal techniques to complete the KB, by learning which types of paths or transitions are indicative of which relation between the start and end points~\cite{lao2011random,Gardner2015EfficientAE}. 
This class of solutions, unfortunately, does not scale well to our setting (cf.~Footnote~\ref{footnote:graph-based}). This appears due, at least in part, to different connectivity characteristics of generics tensors compared to named entity ones such as FB15k~\cite{bordes2013translating}.
Advances in the latter set of methods have led to several embedding-based methods that are highly successful at KB completion for named entities~\cite{nickel2011three,Riedel2013RelationEW,dong2014knowledge,Trouillon2016ComplexEF,Nickel2016ARO}. We compare against many of these, including variants of HolE, TransE, and RESCAL. 

Recent work on incorporating entity type and relation schema in tensor factorization~\cite{krompass2014large,krompass2015type,xierepresentation} has focused on factual databases about named entities, which, as discussed earlier, have very different characteristics than generics tensors.
\citet{nimishakavi2016relation} use entity type information as a matrix in the context of non-negative RESCAL for schema induction on medical research documents. As a byproduct, they complete missing entries in the tensor in a schema-compatible manner. We show that our proposal performs better on generics tensors than their method, SICTF. SICTF, in turn, is meant to be an improvement over the TRESCAL system of \citet{Chang2014TypedTD}, which also incorporates types in RESCAL in a similar manner. Recently, \citet{Schutze2017NoiseMF} proposed a neural model for fine-grained entity typing and for robustly using type information to improve relation extraction, but this is targeted for Freebase style named entities.

For schema-aware discriminative training of embeddings, \citet{xierepresentation} use a flexible ratio of negative samples from both schema consistent and schema inconsistent triples. Their combined ideas, however, do not improve upon vanilla HolE (one of our baselines) on the standard FB15k~\cite{bordes2013translating} dataset. They also consider imposing hierarchical types for Freebase, as entities may have different meanings when they have different types---an issue that typically does not apply to generics KBs.~\citet{komninosfeature} use type information along with additional textual evidence for knowledge base completion on the FB15k237 dataset. They learn embeddings for types, along with entities and relations, and show that this way of incorporating type information has a (small) contribution towards improving performance.
Incorporating given first order logic rules has been explored for the simpler case of matrix factorization~\cite{Rocktschel2015InjectingLB,Demeester2016LiftedRI}. Existing first order logic rule extraction methods, however, struggle to find meaningful rules for generics, making this approach not yet viable in our setting.

\citet{xie2016representation} consider inferring facts about a new entity $\tl{e}$ given a `description' of that entity. They use Convolutional Neural Networks (CNNs) to encode the description, deriving an embedding for $\tl{e}$. Such a description in our context would correspond to knowing some factual triples about $\tl{e}$, which is a restricted version of our active learning setting. 

\citet{krishnamurthy2013low} consider active learning for a particular kind of tensor decomposition, namely CP or Candecomp/Parafac decomposition into a low dimensional space. They start with an \emph{empty} tensor and look for the most informative slices and columns to fill \emph{completely} to achieve optimal sample complexity. Their framework builds upon the \emph{incoherence} assumption on the column space, which does not apply to generics KB.

\citet{Hegde2015AnEA} use an entity-centric information extraction (IE) approach for obtaining new facts about entities of interest. \citet{Narasimhan2016ImprovingIE} use a reinforcement learning approach to issue search queries to acquire additional evidence for a candidate fact. Both of these works, and others along similar lines, are advanced IE techniques that operate via a search for new documents and extraction of facts from them. This is different from the KB completion task, where the only source of information is the starting KB and possibly some details about the involved entities and relations.

\section{Tensors of Generics}
\label{sec:tensor_description}

We consider 
knowledge expressed in terms of \textit{(source, relation, target)} triples, abbreviated as $(s,r,t)$. Such a triple may refer to \textit{(subject, predicate, object)} style facts commonly used in information extraction. Each source and target is an \emph{entity} that is a generic noun, e.g., animals, habitats, or food items. Examples of relations include \emph{foundIn}, \emph{eat}, etc. As mentioned earlier, with each generics triple $(s,r,t)$, we associate a categorical truth value $q \in \{all, some, none\}$, defining the quantification semantics ``$q$ $s$ $r$ (some) $t$''. For instance, ``some animals live in (some) forest'' and ``all dogs eat (some) bone''.
\emph{Given a set $K$ of such triples with annotated truth values, the task is to predict additional triples $K'$ that are also likely to be true.}

In addition to a list of triples, we assume access to \textbf{background information} in the form of entity types and the corresponding \emph{relation schema}, as well as a \emph{taxonomic hierarchy}.\footnote{We do not assume that the schema or taxonomy is perfect, and instead rely on these only for heuristic guidance.} Let $E_T$ denote the set of possible entity types. For each relation $r$, the relation schema imposes a type constraint on the entities that may appear as its source or target. Specifically, using $[\ell]$ to denote the set $\{1, 2, \ldots, \ell\}$, the \emph{schema} for $r$ is a collection $\mathcal{S}_r = \lbrace (\mathcal{D}^{(i)}_r, \mathcal{R}^{(i)}_r) \subseteq E_T \times E_T \mid i \in [\ell]\rbrace$ of domain-range pairs with the following property:
the truth value of $(s,r,t)$ is $none$
whenever for every $i \in [\ell]$ it is the case that $s \notin \mathcal{D}^{(i)}_r$ or $t \notin \mathcal{R}^{(i)}_r$.
For example, the relation \textit{foundIn} may be associated with the schema $\mathcal{S}_{\mathit{foundIn}} = \{$\textit{(animal, location), (insect, animal), (plant, habitat), $\dotsc$}$\}$. Similarly, the taxonomic hierarchy defines a partial order $\mathcal{H}$ over all entities that captures the ``isa'' relation, with direct links such as \textit{isa(dog, mammal)} or \textit{isa(gerbil, rodent)}. We use this information to extract ``siblings'' of a given entity, i.e., entities that share a common parent (this may be easily generalized to any common ancestor).

\section{Guided Knowledge Completion}
\label{sec:IKBC}

We begin with an overview of tensor factorization for KB completion for generics.
%
Let $(s,r,t)$ be a generics triple associated with a categorical quantification label $q \in \lbrace all, some, none \rbrace$. For example, \textit{((cat, havePart, whiskers), all), ((cat, liveIn, homes), some),} and \textit{((cat, eat, bear), none)}. 
Predicting such labels is thus a multi-class classification problem. Given a set $K$ of labeled triples, the goal of tensor factorization is to learn a low-dimensional embedding $h$ for each entity and relation such that some function $f$ of $h$ best captures the given labels. Given a new triple, we can then use $f$ and the learned $h$ to predict the probability of each label for it. $K$ often contains only ``positive'' triples, i.e., those with label \textit{all} or \textit{some}. A common step in discriminative training for $h$ is thus \emph{negative sampling}, i.e., generating additional triples that (are expected to) have label \textit{none}.

With $[m]$ denoting the set $\{1, 2, \ldots, m\}$ as before, let $K = \lbrace(x_i, y_i), i \in [m] \rbrace$ be a set of triples $x_i = (s_i, r_i, t_i)$ and corresponding labels $y_i \in \lbrace 1, 2, 3 \rbrace$ equivalent to categorical quantification label $q_i \in \lbrace all, some, none \rbrace$. We learn entity and relation embeddings $\Theta$ that minimize the multinomial logistic loss defined as:
\begin{align}
& \min_\Theta {\sum_{i=1}^m \sum_{k=1}^3 -\mathbbm{1} \lbrace y_i=k \rbrace \log \Pr(y_i = k \mid x_i, \Theta)} \nonumber \\
\label{eqn:multi-class}
& = \min_\Theta {\sum_{i=1}^m \sum_{k=1}^3 -\mathbbm{1} \lbrace y_i=k \rbrace \log \sigma(y_i\, f(h_r, h_s, h_t))}
\end{align}
where 
$h_r, h_s, h_t \in \mathbb{R}^d$ denote the learned embeddings (latent vectors) for $s, r, t,$ respectively, and $\sigma(\cdot)$ is the sigmoid function defined as $\sigma(z) = \frac{1}{1+\exp(-z)}$.

If the \textit{all} categorical label for generics is unavailable,\footnote{This happens to be the case for current generics KBs, but is expected to change with increasing interest in the research community. A step in this direction is a recent version of the Aristo Tuple KB, http://allenai.org/data/aristo-tuple-kb, which includes \emph{most} as a quantification label, in addition to \emph{some}.} we can simplify the label space to $\lbrace some, none \rbrace$, modeled as $y_i \in \lbrace \pm1 \rbrace$, and reduce the model to binary classification:
\begin{align}
  \label{eqn:two-class}
  \min_\Theta {\sum_{i = 1}^m} \log \left[1+\exp \left[-y_i\, f(h_r, h_s, h_t) \right]\right]
\end{align}

We remark that while this generics task with only two labels appears superficially similar to the standard KB completion task for named entities, the underlying challenges and solutions are different. For instance, the approach of using taxonomic information (as opposed to just entity types) as a guide is uniquely suited to generics KBs. The reason being that a generic entity refers to a \emph{set} of individuals, with a natural subset/superset relation forming a taxonomy, whereas in standard KBs an entity refers to one specific individual. This prevents taxonomy based rules from providing useful information for standard KBs, while our results demonstrate their high value when reasoning with generics. Differences like this lead to differences in what is successful in each setting and what is not. 

While all our proposed schemes are embedding oblivious, for concreteness, we describe and evaluate them for the \textbf{Holographic Embedding} or HolE~\cite{Nickel2016HolographicEO} 
which models the label probability as:
\begin{align}
\label{eqn:HolE}
f(h_r, h_s, h_t) = h_r^\top(h_s \circ h_t)
\end{align}
where $\circ : \mathbb{R}^d \times \mathbb{R}^d \rightarrow \mathbb{R}^d$ denotes circular correlation defined as:
\begin{align}
\label{eqn:circ}
[a \circ b]_k=\sum_{i=0}^{d-1} a_i b_{(i+k)\!\!\mod d}
\end{align}
Intuitively, the $k$-th dimension of circular correlation captures how related $a$ is to $b$ when the dimensions of the latter are shifted (circularly, via the $\mathrm{mod}$ operation) by $k$. In particular $[a \circ b]_0$ is simply the dot product of $a$ and $b$.
As can be deduced from Eqns.~\eqref{eqn:HolE}-\eqref{eqn:circ}, this model resembles circular convolution, but can capture, to some extent, relations that are asymmetric among the source and target entities. This is because $[a \circ b]$ is not the same as $[b \circ a]$ but is rather ``flipped'' ($[a \circ b]_k = [b \circ a]_{d-k}$). If we consider the $d \times d$ matrix $M_{ab}$ of element-wise relationships between $a$ and $b$, the HolE embedding of a relation $r$ between $a$ and $b$ defines a weighted sum of circular anti-diagonals of $M_{ab}$.

Circular correlation can be computed using the fast Fourier transform (FFT), making HolE quite efficient in practice.
\citet{Hayashi2017OnTE} recently showed that HolE and complex embeddings~\cite{Trouillon2016ComplexEF}, which is another state-of-the-art method for KB completion, are equivalent and differ only in terms of constraints on initial values. Further, they proposed a linear time computation for HolE by staying fully within the frequency domain of FFT.

\subsection{Incorporating Types and Relation Schema (\itrs)}
\label{sec:itrs}

As described earlier, relation schema $\mathcal{S}_r$ imposes a restriction on sources and targets that may occur with a relation $r$.
We can incorporate this knowledge both at training and at test times.
Doing this at test time simply translates to relabeling schema inconsistent predicted triples as \textit{none}.
Incorporating this knowledge at training time can be done as a constraint on the random negative samples that the method generates to complement the given, typically positive, triples for training.

In general, the ratio of random negative samples from the entire tensor $\mathcal{T}$ and random negative samples from the schema consistent portion $\mathcal{T}'$ of $\mathcal{T}$ is a parameter that should be tuned such that the resulting negative samples mimic the true distribution of labels. It is worth noting that whether the locally closed world assumption (LCWA) holds or not plays an important role in determining this ratio. However, the idea of mixing the two kinds of negative samples has been used in the literature without considering the nature of the dataset, resulting in some seemingly contradicting empirical results on the optimal ratio~\cite{li2016commonsense,xierepresentation,Shi2017ProjEEP,Xie2017AnIK}. As discussed later, we found sampling from $\mathcal{T}$ to work best on our datasets. 


\subsection{Incorporating Entity Taxonomy (\iet)}
\label{sec:iet}

It is challenging to come up with complex Horn or first order logic rules for generics, as each entity represents a class of individuals that may not all behave identically. However, we can derive simple yet highly effective rules based on categorical quantification labels, leveraging the fact that entities come from different levels in a taxonomy hierarchy.

Let $p$ be the parent entity for entity set $\lbrace c_i \rbrace$. Note that $c_i$ itself is a generic, that is, a class of individuals rather than a single individual. This allows one to make meaningful existential statements such as: if a property holds for all or most members of even one class $c_i$, then it holds for some (reasonable number of) members of its parent class $p$. We use the following rules:\footnote{The last rule may not be appropriate for KBs where $\mathit{some}$ may refer to the extreme case of a single individual. This is not the case for the KBs we use for our evaluation.}
\begin{align*}
((p,r_j,t_j),\mathit{all})
  & \Rightarrow \forall i~~ ((c_i,r_j,t_j),\mathit{all}) \\
\forall i~~ ((c_i,r_j,t_j),\mathit{all})
  & \Rightarrow ((p,e_j,t_j),\mathit{all}) \\
\exists i ~~((c_i,r_j,t_j),\mathit{all})
  & \Rightarrow ((p,e_j,t_j),\mathit{some}) \\
\exists i ~~ ((c_i,r_j,t_j),\mathit{some})
  & \Rightarrow ((p,e_j,t_j),\mathit{some}) 
\end{align*}

We apply these rules to address sparsity of generics tensors, making tensor factorization more robust. Specifically, given initial triples $K$, we use applicable rules to derive additional triples $K'$, perform tensor factorization on $K \cup K'$, and then revisit the triples in $K'$ using their predicted label probabilities. Note that this approach allows us to be robust to taxonomic errors: instead of assuming each triple in $K'$ is true, we use this only as a prior and let tensor factorization determine the final prediction based on global patterns it finds.


\section{Active Learning for New or Rare Entities}
\label{sec:AL}



To address the incomplete nature of generics KBs, we consider \emph{rare} entities for which we have very few facts, or \emph{new} entities which are present in the taxonomy but for which we have no facts in the KB.
The goal is to use tensor factorization to generate high quality facts about such entities.

For instance, consider the task of inferring facts about \textit{oriole}, where \emph{all we know is that it is a bird}. We assume a restricted budget on the number of facts we can query (for human annotation) about \textit{oriole}, using which we would like to predict many more high-quality facts about it.

Given a fixed query budget $B$, what is the optimal set of queries we should generate for human annotation about a new or rare entity $\tl{e}$ for this task? We view this as an \emph{active learning} problem and propose a two-step algorithm. First, we use taxonomy guided uncertainty sampling to propose a list $L$ to potentially query. Next, we describe a submodular objective function and a corresponding linear time algorithm to choose an optimal subset $\widehat{L} \subseteq L$ satisfying $|\widehat{L}| = B$. We then use $\widehat{L}$ for human annotation, append the result to the original KB, and perform tensor factorization to predict additional new facts about $\tl{e}$. For notational simplicity and without loss of generality, throughout this section, we consider the case where $\tl{e}$ appears as the source entity in the triple; the ideas apply equally when $\tl{e}$ appears as the target entity in the triple. 


\subsection{Knowledge Guided Uncertainty Quantification}
\label{sec:AL1}


We now discuss the active learning and specifically uncertainty sampling method we use to propose a list of triples to query. Uncertainty sampling considers the uncertainty for each possible triple $(\tl{e}, r_i, e_i)$, defined as how far away from 0.5 is the conditional probability of this fact given the facts we already know (i.e.,
KB)~\cite{Settles2012ActiveL}. The question is how to model this conditional probability. A simple baseline is to consider \textbf{Random} queries, i.e.,  $r, e$ are  selected randomly from the list of relations and entities in the tensor, respectively.

To infer information about $\tl{e}$, we propose the following approximation for the conditional probability of a new fact about $\tl{e}$ given the KB. Let  $\tl{E}_{\tl{e}} = \lbrace e \mid \corr(\tl{e}, e) > 0 \rbrace$ be the set of entities that are correlated with $\tl{e}$, $\Omega = \lbrace ((e_i, r_i, e'_i),y_i) \mid e_i \in \tl{E}_{\tl{e}} \rbrace$ be the set of known facts about such entities, and $y_i$ be the label for the triple $(e_i, r_i, e'_i)$. We have:
\begin{align}
\label{eqn:cond_pr_approx}
\Pr(f(h_{r_i},h_{\tl{e}},h_{e'_i})) \simeq \frac{1}{|\Omega|} \sum_{e_i \in \tl{E}_{\tl{e}}} \corr(\tl{e}, e_i) \,\, y_i
\end{align}
However, in practice, we cannot measure $\corr(\tl{e}, e_i)$ for every entry in the KB as we do not have complete information about $\tl{e}$. One simple idea is to consider that every entity is correlated with $\tl{e}$: $\corr(\tl{e}, e_i) = 1 ~~\forall e_i \in E$. We will refer to this as \textbf{Schema Consistent} query proposal as this relates to summing over all possible (hence schema consistent) facts.

Since we have access to taxonomy information, we can do a more precise, \textbf{Sibling Guided}, approximation.\footnote{One may also define $\corr$ based on entity similarity in a distributional space. One challenge here is that such similarity generally doesn't preserve types. E.g., \emph{dog} may co-occur more often with and thus be ``closer" to \emph{bone} or \emph{barking} in a distributional space, than to siblings such as \emph{cat} or other pet animals, which are more helpful in our setting.} We propose the following approximation for $\corr(\tl{e}, e_i)$ for $e_i \in E$:
\begin{align}
\label{eqn:corr_approx}
\corr(\tl{e}, e_i)=\left\lbrace \begin{array}{ll}
1 & \text{if\ } e_i \in \text{sibling}(\tl{e}) \\
0 & \text{otherwise} \end{array} \right.
\end{align}
Eqns.~\eqref{eqn:cond_pr_approx} and~\eqref{eqn:corr_approx} can be used to infer uncertain triples: if every sibling of $\tl{e}$ has relationship $r$ with an entity $e'$, we can infer for ``free'' that this is the case for $\tl{e}$ as well. On the other hand, when 
siblings disagree in this respect,
there is more uncertainty about $(\tl{e},r,e') $ (according to~\eqref{eqn:cond_pr_approx} and~\eqref{eqn:corr_approx}), making this triple a good candidate to query.
In our example of \textit{oriole}, the siblings are the \textit{birds} that exist in the tensor, e.g., \textit{hummingbird, finch, woodpecker, etc}. All of them \textit{(eat, insect)} and hence we infer this for oriole. But there is no agreement on \textit{(appearIn, farm)} and hence this is added to the query list.

Algorithm~\ref{algo:AL} formalizes this process. Setting some upper ($\tau_U$) and lower ($\tau_L$) bounds on the conditional probability (Eqn.~\eqref{eqn:cond_pr_approx}) which quantifies the uncertainty, we reach a set $L = \lbrace (\tl{e}, r_i, e_i), i \in I \rbrace$ of triples to query. Using another high threshold $\kappa_M > \tau_U$, we also infer the set $M = \lbrace (\tl{e}, r_j, e_j), j \in J \rbrace$ of triples that a large majority of siblings agree upon, and hence $\tl{e}$ is expected to agree with as well. Triples whose conditional probability estimate is between $\kappa_M$ and $\tau_U$ are considered neither certain enough to include in $M$ nor uncertain enough to justify adding to $L$ for human annotation in hopes of learning from it. Similarly, triples with a conditional probability estimate lower than $\tau_L$ are discarded. The output of Algorithm~\ref{algo:AL} is the list $L$ to query and the list $M$ to add directly to the knowledge base.

\begin{algorithm}[t]
\caption{Active Learning for Query Proposal}
\label{algo:AL}
\begin{algorithmic}[1]
\renewcommand{\algorithmicrequire}{\textbf{input}}
\renewcommand{\algorithmicensure}{\textbf{output}}

\newcommand{\LineIf}[2]{ \STATE \algorithmicif\ {#1}\ \algorithmicthen\ {#2}}

\REQUIRE  new entity $\tl{e}$, KB, taxonomy, lower bound $\kappa_M$ on agreement, lower bound $\tau_L$ on uncertainty, upper bound $\tau_U$ on uncertainty
\STATE extract list $S_{\tl{e}}$ of sibling$(\tl{e})$ using taxonomy
\STATE for each $e_i \in S_{\tl{e}}$, add all facts about $e_i$ to $\Omega$
\FOR{$(\tl{e}, r_i, e'_i) \in \Omega$}
\STATE use~\eqref{eqn:cond_pr_approx}-\eqref{eqn:corr_approx} to estimate
  $\Pr(f(h_{r_i},h_{\tl{e}},h_{e'_i}))$
\LineIf{$p \geq \kappa_M$}{add $(\tl{e}, r_i, e'_i)$ to M}
\LineIf{$\tau_L \leq p \leq \tau_U$}{add $(\tl{e}, r_i, e'_i)$ to L}
\ENDFOR

\ENSURE L, M
\end{algorithmic}
\end{algorithm}
\subsection{Efficient Subset Selection}
\label{sec:AL2}

 Given the list $L$ as above (Algorithm~\ref{algo:AL}), which we can write in short as $L = \lbrace (r_i, e_i), i \in I \rbrace$, the problem is to find the ``best'' subset $\widehat{L}$. A baseline for such a selection is to choose the top k queries. We will refer to this as \textbf{TK} subset selection.
 
Viewing subset selection as a combinatorial problem, we devise an objective $\mathcal{F}$ that models several natural properties of this subset. We then prove that $\mathcal{F}$ is \textbf{submodular}, that is, the marginal gain in $\mathcal{F}(L)$ obtained by adding one more item to $L$ decreases as $L$ grows.\footnote{Formally, for $L'' \subseteq L' \subseteq L$ and for $l = (r_l, e_l) \in L \setminus L'$, we have
  $\mathcal{F}(L'' \cup l) - \mathcal{F}(L'') \geq  \mathcal{F}(L' \cup l) - \mathcal{F}(L')$.
}
Importantly, this implies that there is a simple known greedy algorithm that can efficiently compute a worst-case $(1- 1/e)$-approximation of the global optimum of $\mathcal{F}$~\cite{nemhauser1978analysis}. We refer to this as \textbf{SM} subset selection.

Since queried samples will eventually be fed into tensor factorization, we would like $\widehat{L}$ to \textit{cover} entities (for the other argument of the triple) and relations as much as possible. In addition, we would like $\widehat{L}$ to be \textit{diverse}, i.e., prioritize relations and entities that are more varied.\footnote{This agrees with the sampling method of~\citet{chen2014coherent} for factorizing coherent \textit{matrices} with missing values, which chooses samples with probability proportional to their local coherence.}  At the same time, we would also want to minimize redundancy, i.e., avoid choosing relations (entities) that are too similar. Let $\mathcal{F}(\widehat{L}, R_{\widehat{L}}, E_{\widehat{L}})$ denote our objective, where $R_{\widehat{L}}$, $E_{\widehat{L}}$ is the set of relations and entities in $\widehat{L}$, respectively. We decompose it as:
  \begin{align}
  \label{eqn:objective}
  \ifthenelse{\isundefined{\forarxiv}}{
 &\mathcal{F}(\widehat{L}, R_{\widehat{L}}, E_{\widehat{L}}) = w_{\mathcal{C}}\mathcal{C}(\widehat{L}, R_{\widehat{L}}, E_{\widehat{L}})  \\ \nonumber
 &~~~~ + w_{\mathcal{D}}\mathcal{D}(\widehat{L}, R_{\widehat{L}}, E_{\widehat{L}}) - w_{\mathcal{R}}\mathcal{R}(\widehat{L}, R_{\widehat{L}}, E_{\widehat{L}})
 }{
 \mathcal{F}(\widehat{L}, R_{\widehat{L}}, E_{\widehat{L}}) & = w_{\mathcal{C}}\mathcal{C}(\widehat{L}, R_{\widehat{L}}, E_{\widehat{L}})
 + w_{\mathcal{D}}\mathcal{D}(\widehat{L}, R_{\widehat{L}}, E_{\widehat{L}}) - w_{\mathcal{R}}\mathcal{R}(\widehat{L}, R_{\widehat{L}}, E_{\widehat{L}})
 }
 \end{align}
 where the terms in RHS correspond to \textit{coverage}, \textit{diversity}, and \textit{redundancy}, resp., and $w_{\mathcal{C}}, w_{\mathcal{D}},w_{\mathcal{R}}$ are the corresponding non-negative weights. Next, we propose functional forms for these terms. Note that any function that captures the described properties can be used instead, as long as the objective remains submodular. 
 
Let $R$ and $E$ denote the set of relations and entities in the KB, resp. The coverage simply captures the fraction of entity and relations that we have included in $\widehat{L}$:
 \begin{align*}
\mathcal{C}(\widehat{L}, R_{\widehat{L}}, E_{\widehat{L}}) = \frac{|R_{\widehat{L}}|}{|R|} +  \frac{|E_{\widehat{L}}|}{|E|}.  
\end{align*}
The diversity for $\widehat{L}$ is the sum of the diversity measure of the entities and relations included in the set:
\begin{align*}
\mathcal{D}(\widehat{L}, R_{\widehat{L}}, E_{\widehat{L}})  = \underset{(r, e) \in \widehat{L}}{\sum} \left[ V_r + V_e  \right],
\end{align*}
\begin{align*}
V_r &= \frac{|E_{S_r}|+|E_{T_r}|}{|E|},
 &
 V_{e} &= \frac{|R_{e}|+|E_{S_e}|}{|R| + |E|}.  
 \end{align*}
Here $V_r$ and $V_e$ represent the diversity measure of relation $r$ and entity $e$, respectively. We use $E_{S_r}, E_{T_r}$ to denote the set of sources and targets that appear for relation $r$ in the KB, $R_e$ as the set of relations in the KB that have $e$ as their target, and $E_{S_e}$ as the set of entities that appear as the first entity when $e$ is the second entity of the triple in the KB. The diversity measure for each relation $r$ is defined as the ratio of the number of entities that appear in the KB as its source or target, over the total number of entities. Similarly, for an entity $e$, its diversity is defined as the ratio of the number of relations involving $e$ plus the number of source entities that co-occur with $e$ in a relation, over the total number of relations and entities. Note that the diversity measure is an intrinsic characteristic of each entity and relationship, dictated by the KB and independent of the set $L$, and can thus be computed in advance.

 
 As described above, redundancy is a measure of similarity between relations(entities) in $\widehat{L}$. Tensor factorization yields an embedding for each relation(entity) given the facts they participated in. Therefore, the learned embeddings are one of the best options for capturing similarities. Let $h_e$ (and $h_r$) denote the learned embedding for entity $e$ (and relation $r$, resp.). We define
 \begin{align*}
\ifthenelse{\isundefined{\forarxiv}}{
  \mathcal{R}(\widehat{L}, R_{\widehat{L}}, E_{\widehat{L}})
    & = \underset{\!r_1,r_2 \in \widehat{L}}{\!\!\!\!\!\sum} \!\! \Vert h_{r_1}-h_{r_2} \Vert \\
    & \ \ \ \ \ + \underset{e_1,e_2 \in \widehat{L}}{\!\!\!\!\!\sum} \!\! \!\Vert h_{e_1}-h_{e_2} \Vert.
}{
  \mathcal{R}(\widehat{L}, R_{\widehat{L}}, E_{\widehat{L}})
    & = \underset{\!r_1,r_2 \in \widehat{L}}{\!\!\!\!\!\sum} \!\! \Vert h_{r_1}-h_{r_2} \vert
      + \underset{e_1,e_2 \in \widehat{L}}{\!\!\!\!\!\sum} \!\! \!\Vert h_{e_1}-h_{e_2} \Vert.
}
 \end{align*}
 
This completes the definition of all pieces of our objective function, $\mathcal{F}$, from Eqn.~\eqref{eqn:objective}. In Algorithm~\ref{algo:subset_selection}, we present our efficient greedy method to select a subset of $L$ that approximately optimizes $\mathcal{F}$.

\begin{algorithm}[t]
\caption{Query Subset Selection}
\label{algo:subset_selection}
\begin{algorithmic}[1]
\renewcommand{\algorithmicrequire}{\textbf{input}}
\renewcommand{\algorithmicensure}{\textbf{output}}
\REQUIRE KB, budget $B$, query list $L$ from Alg.~\ref{algo:AL}.
\STATE $\forall (r,e) \in L$, compute the diversity measure $V_r, V_e$
\STATE $\widehat{L} \leftarrow \emptyset$
\FOR{$j=1$ to $B$}
\ifthenelse{\isundefined{\forarxiv}}{
  \STATE $\forall l \in L \setminus \widehat{L}: \mathcal{G}(l) = \mathcal{F}(\widehat{L} \cup l) - \mathcal{F}(\widehat{L})$,\\ for $\mathcal{F}$ in~\eqref{eqn:objective}
}{
  \STATE $\forall l \in L \setminus \widehat{L}: \mathcal{G}(l) = \mathcal{F}(\widehat{L} \cup l) - \mathcal{F}(\widehat{L})$, for $\mathcal{F}$ in~\eqref{eqn:objective}
}
\STATE Select $l^* = \arg\max_{L \setminus \widehat{L}} \mathcal{G}(l)$
\STATE Add $l^*$ to $\widehat{L}$
\ENDFOR

\ENSURE $\widehat{L}$
\end{algorithmic}
\end{algorithm}

Despite being a greedy approach that simply adds the currently most valuable single query to $\widehat{L}$ and repeats, the submodular nature of $\mathcal{F}$, which we will prove shortly, guarantees that Algorithm~\ref{algo:subset_selection} provides an approximation that, even in the worse case, is no worse than a factor of $1 - 1/e$ from the (unknown) true optimum of $\mathcal{F}$. This is formalized in the following theorem.
Since addition preserves submodularity and the weights $w_{\mathcal{C}}, w_{\mathcal{D}}, w_{\mathcal{R}}$ are non-negative, we will show that each of the three terms in $\mathcal{F}$ is submodular.

 \begin{theorem}
 \label{thm:greedy}
Given a tensor KB, a budget $B$, and a candidate query list $L$, the quality $\mathcal{F}(\widehat{L}, R_{\widehat{L}}, E_{\widehat{L}})$ of the output $\widehat{L}$ of Algorithm~\ref{algo:subset_selection} is a $(1 - 1/e)$-approximation of the global optimum of $\mathcal{F}$.
 \end{theorem}




\begin{proof}
 In order to prove the result, in suffices to show that $ \mathcal{F}(\widehat{L}, R_{\widehat{L}}, E_{\widehat{L}})$ in Equation~\eqref{eqn:objective} is submodular~\cite{nemhauser1978analysis}.
 %
 To this end, we show that for $L'' \subseteq L' \subseteq L$ and for $l = (r_l, e_l) \in L \setminus L'$,
  \begin{small}
  \begin{align*}
  \mathcal{F}(L'' \cup l) - \mathcal{F}(L'') \geq  \mathcal{F}(L' \cup l) - \mathcal{F}(L') 
  \end{align*}
  \end{small}
  Since addition preserves submodularity and the weights $w_{\mathcal{C}}, w_{\mathcal{D}},w_{\mathcal{R}}$ are non-negative, it suffices to show that each term in $\mathcal{F}$ is submodular.
  
First, consider the \emph{coverage} term, $\mathcal{C}(\widehat{L}, R_{\widehat{L}}, E_{\widehat{L}})$. In order to prove that it is submodular, we verify:
  \begin{small}
  \begin{align*}
  \frac{\left(|R_{L'' \cup l}| - |R_{L''}|\right)}{|R|} &\geq   \frac{\left(|R_{L' \cup l}| - |R_{L'}|\right)}{|R|} \\
    \frac{\left(|E_{L'' \cup l}| - |E_{L''}|\right)}{|E|} &\geq   \frac{\left(|E_{L' \cup l}| - |E_{L'}|\right)}{|E|}
  \end{align*}
  \end{small}
%
Note that for the numerators of each of the above lines, the difference can be either $+1$ or $0$. Since $L'' \subset L'$, LHS is, by definition, never less than RHS and the inequalities holds.

Next, consider the diversity term, $\mathcal{D}(\widehat{L}, R_{\widehat{L}}, E_{\widehat{L}})$. The above argument directly applies here as well.

Finally, consider the \emph{redundancy} term. In order to show that $-\mathcal{R}(\widehat{L}, R_{\widehat{L}}, E_{\widehat{L}})$ is submodular, note that when taking the difference between $\mathcal{R}(L'' \cup l)$ and $\mathcal{R}(L'') $ the terms that correspond to both entities (or both relations) being in $L''$ cancel out. The same holds for $\mathcal{R}(L' \cup l) - \mathcal{R}(L')$. We thus have:
\ifthenelse{\isundefined{\forarxiv}}{
  \begin{small}
  \begin{align*}
  & \mathcal{R}(L'' \cup l) - \mathcal{R}(L'') = \\
  & \ \ \underset{r_l \in l,r_2 \in {L''}}{\sum}  \Vert h_{r_1}-h_{r_2} \Vert + \!\!\! \underset{e_l \in l,e_2 \in {L''}}{\sum}  \Vert h_{e_1}-h_{e_2} \Vert \\
  &\mathcal{R}(L' \cup l) - \mathcal{R}(L') = \\
  & \ \ \underset{r_l \in l,r_2 \in {L'}}{\sum}  \Vert h_{r_1}-h_{r_2} \Vert + \!\!\! \underset{e_l \in l,e_2 \in {L'}}{\sum}  \Vert h_{e_1}-h_{e_2} \Vert
  \end{align*}
  \end{small}
}{
  \begin{align*}
  \mathcal{R}(L'' \cup l) - \mathcal{R}(L'') & =
  \quad\underset{r_l \in l,r_2 \in {L''}}{\sum}  \Vert h_{r_1}-h_{r_2} \Vert + \underset{e_l \in l,e_2 \in {L''}}{\sum}  \Vert h_{e_1}-h_{e_2} \Vert, \\
  \mathcal{R}(L' \cup l) - \mathcal{R}(L') & =
  \quad  \underset{r_l \in l,r_2 \in {L'}}{\sum}  \Vert h_{r_1}-h_{r_2} \Vert + \underset{e_l \in l,e_2 \in {L'}}{\sum}  \Vert h_{e_1}-h_{e_2} \Vert.
  \end{align*}
}
Since $L'' \subseteq L'$ and norms are non-negative, 
$$\mathcal{R}(L'' \cup l) - \mathcal{R}(L'')  \leq \mathcal{R}(L' \cup l) - \mathcal{R}(L').$$
The reverse inequality holds for the negation of both sides, proving that $-\mathcal{R}(\widehat{L}, R_{\widehat{L}}, E_{\widehat{L}})$ is submodular.

Combining the three items concludes the proof.
 \end{proof}

%


We will complement this theoretical guarantee in the experiments section (cf.~Table~\ref{table:AL-summary}) by empirically comparing the performance of our query proposal and subset selection methods with baselines.


\section{Experiments} \label{sec:experiments}

We begin with a description of the datasets and the general setup, then evaluate the effectiveness of our guided KB completion approach, and end with an evaluation of our active learning method.\footnote{Data and code available from the authors.}

\subsection{Dataset and Setup}

To assess the quality of our guided KB completion method, we consider the only large existing knowledge bases about generics that we are aware of:

\begin{enumerate}

\item A \textbf{Science} tensor containing
facts about various scientific activities, entities (e.g., animals, instruments, body parts), units, locations, occupations, etc.~\cite{dalvi2017domaintargeted}.\footnote{Aristo Tuple KB v0, http://allenai.org/data/aristo-tuple-kb.} This starting tensor has a precision of about $80\%$ and acts as a valuable resource for challenging tasks such as question answering. Our goal is to start with this tensor and infer more scientific facts at a similar or higher level of precision.

\item An \textbf{Animals} sub-tensor of the Science tensor, which focuses on facts about animals and also has a similar starting precision. Again, the goal is to infer more facts about animals.

\end{enumerate}

The mainstream approach for KB completion is to focus on entities that are mentioned sufficiently often. For instance, the commonly used FB15K dataset guarantees that every entity appears at least 100 times. As a milder version of this, we focus on the subset of the starting tensors where every entity appears at least 20 times. The resulting statistics of the tensors we use here are shown in Table~\ref{table:dataset}.

\begin{table}[tb]
\setlength{\tabcolsep}{8pt}
\setlength{\doublerulesep}{\arrayrulewidth}
\centering
\small
\begin{tabular}{cccc}
Dataset & \# Entities & \# Relations & \# Triples \\
\hline \hline \Tstrut
Animals &  224 & 129 & 10,604 \\
Science & 1,255 & 1,513 & 66,643 \\
\hline
\end{tabular}
\caption{Datasets, with a 3/1/1 train/validation/test split.}
\label{table:dataset}
\end{table}

This data, which is the only one we are aware of with generics, does not include $((s,r,t), all)$ style triples. We therefore use the objective function in Eqn.~\eqref{eqn:two-class} rather than the multi-class one in Eqn.~\eqref{eqn:multi-class}. Despite this limitation of the dataset and its superficial similarity to the binary classification task underlying standard (non-generics) KB completion, our results reveal that extending a generics KB is surprisingly difficult for existing methods.

\citet{dalvi2017domaintargeted} use a pipeline consisting of Open IE~\cite{Banko2007OpenIE} extractions, aggregation, and clean up via crowd-sourcing to generate the Science tensor. These facts come with a relevant WordNet~\cite{miller1995wordnet} based taxonomy, entity types (derived from WordNet `synsets'), and relation schema. Our method capitalizes on this additional information\footnote{In order to limit potential error propagation, we
collapse the taxonomy to the top two levels
in our experiments.} to perform high quality knowledge completion.


Our \textbf{evaluation metric} is the accuracy of the top $k$ triples generated by various KB completion methods. We also visualize entire precision-recall curves, where possible. While this metric requires human annotation and is thus more cumbersome than fully-automatic metrics, it is arguably more suitable for evaluating \emph{generative tasks} with a massive output space, such as KB completion. In this setting, evaluation against a relatively small held out test set can be misleading---a method may be highly accurate at generating thousands of valid and useful triples even if it does not necessarily classify specific held out instances accurately. While measures such MAP and MRR have been used in the past to alleviate this, they provide only a partial solution to the inherent difficulty of evaluating generative systems. Annotation-efficient evaluation methods have recently been proposed to address this challenge~\cite{Sabharwal2017HowGA}.

\subsection{Guided KB Completion}
%
We first compare our method (Section~\ref{sec:IKBC}) with existing KB completion techniques on the Animals tensor, and then demonstrate that its effectiveness carries over scalably to the larger Science tensor as well. In what follows, $\mathcal{T}$ denotes the tensor under consideration.

We examine two alternatives for generating negative samples: given a triple $(s,r,t) \in \mathcal{T}$, replace $s$ with (1) any entity $s'$ or (2) an entity $s'$ of the same type as $s$. The resulting perturbed triple $(s',r,t)$ is then treated as a negative sample if it is not present in $\mathcal{T}$. We also considered a weighted combination of (1) and (2), and found random sampling to be the most reliable on our datasets. This complies with the commonly used LCWA assumption not being applicable to these tensors.

\begin{figure}[tb]
\centering
\includegraphics[width=\plotwidth, trim={10 0 10 7ex}, clip]{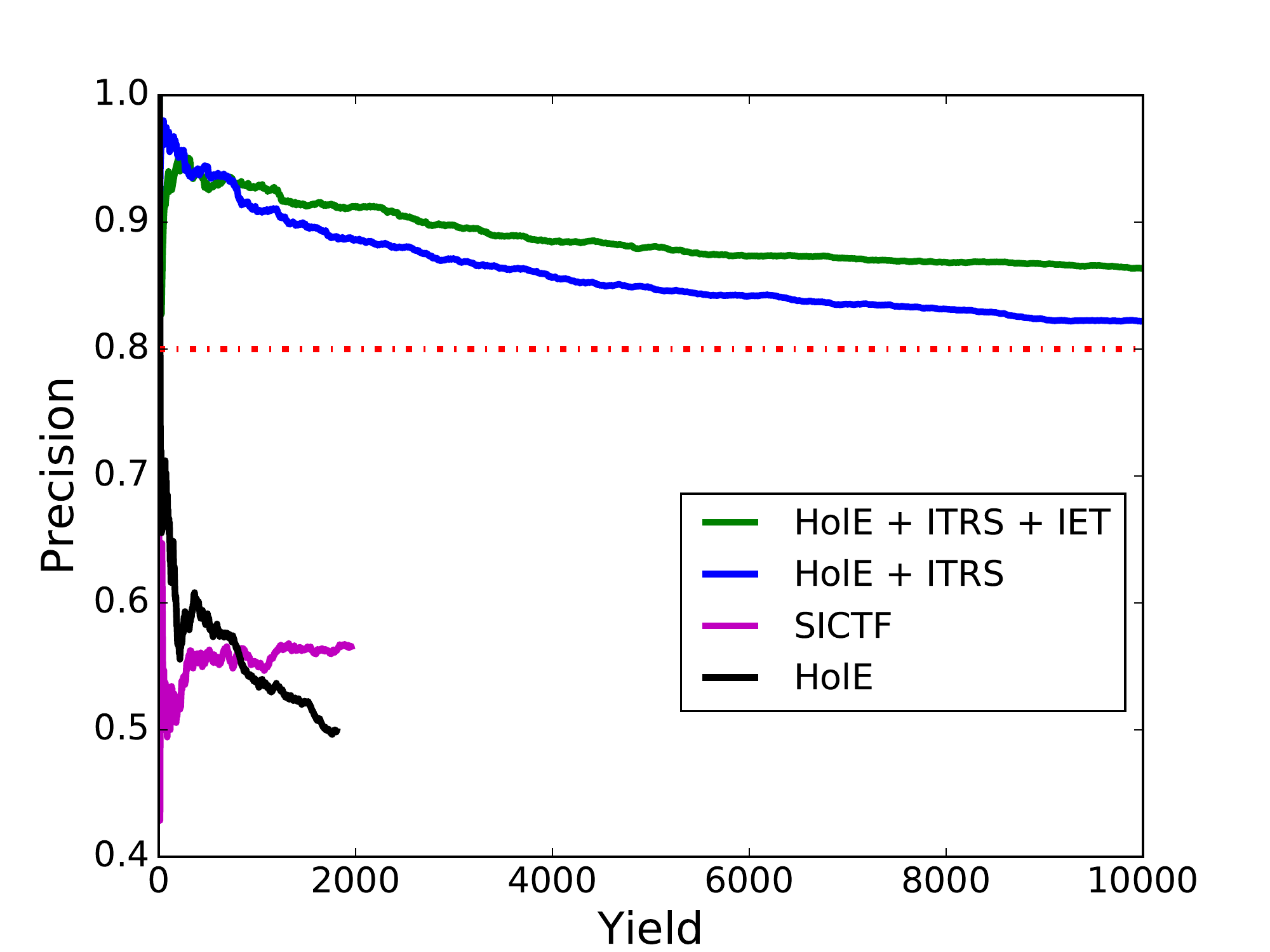} 
\caption{\label{fig:tensor_factorization} Precision-yield curves for various embedding-based methods on the Animals tensor. State-of-the-art named-entity inspired approaches (black, pink) have low precision even at a low yield. TransE is omitted due to its very low precision here, around 10\%. Our method (HolE$+$ITRS$+$IET, green) doubles the size of the starting tensor at a precision of 86.4\%.}
\end{figure}

\begin{table*}[tb]
\setlength{\tabcolsep}{2ex}
\setlength{\doublerulesep}{\arrayrulewidth}
\centering
\small
\newcommand\invalid[1]{\underline{#1}}
\begin{tabular}{lll}
\multicolumn{1}{l}{\textbf{HolE}} & \multicolumn{1}{l}{\textbf{SICTF}} & \multicolumn{1}{l}{\textbf{HolE$+$ITRS$+$IET}} \\
\hline \hline
	\invalid{farm,join,farm}	&	\invalid{penguin,has part,tooth}	&	salmon,thrive in,water	\\
	family,join,family	&	\invalid{mosquito,spread,parasite}	&	animal,give birth to,animal	\\
	\textit{tree,resemble,tree}	&	spider,has part,skin	&	duck,feed in,water	\\
	\textit{water,is known as,water}	&	\invalid{elephant,eat,fish}	&	fish,migrate to,water	\\
	\invalid{virus,attract,virus}	&	shark,has part,skin	&	fish,thrive in,water	\\
	\textit{animal,resemble,animal}	&	crab,eat,insect	&	turtle,swim in,water	\\
	\textit{tree,is known as,tree}	&	snake,eat,fish	&	salmon,swim in,water	\\
	\textit{habitat,is known as,habitat}	&	otter,has part,tooth	&	turtle,live in,water	\\
	\textit{envment.,is known as,envment.}	&	\invalid{meat,attract,hummingbird}	&	animal,chew,food	\\
	\textit{man,join,man}	&	\invalid{spider,has part,claw}	&	insect,destroy,tree	\\
	bird,give birth to,bird	&	\invalid{turtle,has part,tooth}	&	farm,possess,horse	\\
	\textit{region,is known as,region}	&	human,eat,plant	&	fish,swim in,ocean	\\
	\invalid{virus,derive from,virus}	&	\invalid{monkey,has part,wing}	&	turtle,feed in,water	\\
	\textit{food,resemble,food}	&	\invalid{dolphin,has part,tooth}	&	turtle,float in,water	\\
	\textit{bird,is known as,bird}	&	carnivore,live in,water	&	dinosaur,walk on,leg	\\
	\textit{field,resemble,field}	&	lizard,eat,fish	&	turtle,migrate to,water	\\
	\textit{fish,is known as,fish}	&	\invalid{pelican,has part,tooth}	&	turtle,return to,water	\\
	\textit{bird,resemble,bird}	&	\invalid{caterpillar,turn into,bird}	&	\invalid{man,ride,cattle}	\\
	\invalid{grass,graze in,man}	&	bee,pollinate,garden	&	turtle,swim in,ocean	\\
	\textit{animal,is known as,animal}	&	virus,infect,bird	&	fish,float in,ocean	\\
\hline
\end{tabular}
\caption{\label{table:top-predictions} Top 20 predictions by various methods, with invalid triples \underline{underlined} and uninteresting ones , such as (X, is known as, X) or (Y, resembles, Y), shown in \textit{italics}. While some of this assessment can be subjective, it is evident that our method, HolE$+$ITRS$+$IET, generates many more triples that are valid and interesting than competing approaches.
}
\end{table*}

As \textbf{baselines}, we  consider extensions of three state-of-the-art embedding-based KB completion methods: HolE, TransE, and RESCAL. As mentioned earlier, two leading graph-based methods, SFE and PRA, did not scale well. Both vanilla TransE and RESCAL resulted in poor performance; we thus report numbers only for their extensions. Specifically, we consider 3 baselines: (1) HolE, (2) TransE$+$Schema, and (3) SICTF which extends RESCAL and incorporates schema.

\begin{table*}[htb]
\setlength{\tabcolsep}{6pt}
\setlength{\doublerulesep}{\arrayrulewidth}
\centering
\small
\begin{tabular}{|cc|cccc|}
\hline
 & & \multicolumn{4}{c|}{\# New True Triples Inferred} \\ \cline{3-6} \Tstrut
 Query & Subset & From & Sibling & Tensor & \\
 Proposal & Selection & Anntation & Argument & Factorization & Total \\
\hline \hline \Tstrut
Random & - & 0 & - & 0 & 0 \\
Schema Consistent & TK & 73 & - & 10 & 83 \\
Schema Consistent & SM & 57 & - & 27 & 84 \\
Sibling Guided & TK & 96 & 17 & 211 & 324 \\
\textbf{Sibling Guided} & \textbf{SM} & 100 & 17 & 366 & \textbf{483} \\
\hline
\end{tabular}
\caption{Active Learning for new entities: Number of new facts inferred (from annotation, sibling agreement, tensor factorization, and in total) for a representative new entity $\tl{e}$, when querying 100 facts about $\tl{e}$ for human annotation.}
\label{table:AL-summary}
\end{table*}

Figure~\ref{fig:tensor_factorization} shows the resulting precision-yield curves for the predictions made by each method on the \textbf{Animals} dataset containing 10.6K facts. Specifically, for each method, we rank the predictions based on the method's assigned score and compute the precision of the top $k$ predictions for varying $k$. As expected, we observe a generally decreasing trend as $k$ increases.  TransE$+$ITRS gave a precision of only around 10\% and is omitted from the plot. We make two observations:

First, deriving new facts for these generics tensors at a high precision is challenging! Specifically, none of the baseline methods (black and pink curves), which represent state of the art for named-entity tensors, achieve a yield of more than 10\% of $\mathcal{T}$ (i.e., 1K predictions) even at a precision of just 60\%.

Second, external information, if used appropriately, can be surprisingly powerful in this setting. Specifically, simply incorporating relation schema (ITRS, blue curve) allows HolE-based completion to double the size of the starting tensor $\mathcal{T}$ by producing over 10K new triples at a precision of 82\%. Further, incorporating entity taxonomy (IET, green curve) to address tensor sparsity results in the same yield at a statistically significantly higher precision of 86.4\%.

It turns out that not only does our method result in substantially improved PR curves, it also generates \textbf{qualitatively more interesting} and useful generic facts about the world than previous methods. We illustrate this in Table~\ref{table:top-predictions}, which lists the top 20 predictions made by various approaches. The triples shown in red are false predictions (e.g., \emph{(penguin, has part, tooth)}, \emph{(grass, graze in, man)}, \emph{(caterpillar, turn into, bird)}) or uninteresting ones (e.g., \emph{(water, is known as, water)}). As we see, a vast majority of the top 20 predictions made by both vanilla HolE and SICTF fall into these categories. On the other hand, our method, HolE$+$ITRS$+$IET, predicts 19 true tripes out of the top 20, including interesting scientific facts that were evidently missing from the starting tensor, such as \emph{(salmon, thrive in, water)}, \emph{(fish, swim in, ocean)} and \emph{(insect, destroy, tree)}.

Finally, we evaluate our proposal on the entire \textbf{Science} dataset with 66.6K facts. Since graph-based methods did not scale well to the much smaller Animals dataset and other methods performed substantially worse there, we focus here on the scalability and prediction quality of our method. We found that HolE$+$ITRS$+$IET scales well to this high dimension, \emph{doubling the number of facts} by adding 66K new facts at $74\%$ precision.
Although the Science tensor is 1,000 times larger than the Animals tensor, the method took only 10x longer to run (3 minutes on Animals tensor vs.\ 56 minutes on Science tensor, using a 2.8GHz, 16GB Macbook Pro). With additional improvements such as parallelization, it is easily possible to further scale the method up to substantially larger tensors.

\subsection{Active Learning for New Entities}

To assess the quality of our active learning mechanism (Section~\ref{sec:AL}), we consider predicting facts about a new entity $\tl{e}$ that is not in the Animals tensor. For illustration, we choose $\tl{e}$ from
the Science tensor vocabulary while ensuring that it is present in the WordNet taxonomy.

The setup is as follows. We first use a \emph{query generation mechanism} (Random, Schema Consistent, or Sibling Guided; cf. Section~\ref{sec:AL1}) to propose an ordered list $L$ of facts about $\tl{e}$ to annotate. Next, we perform \emph{subset selection} (Top $k$ or TK, Submodular or SM; cf.~Section~\ref{sec:AL2}) on $L$ to identify a subset $\widehat{L}$ of up to $100$ most promising queries. These are then annotated and the true ones fed into tensor factorization as additional input to infer further new facts about $\tl{e}$.

\begin{figure}[tb]
\centering
\includegraphics[width=\plotwidth, trim={0 0 0 0ex}]{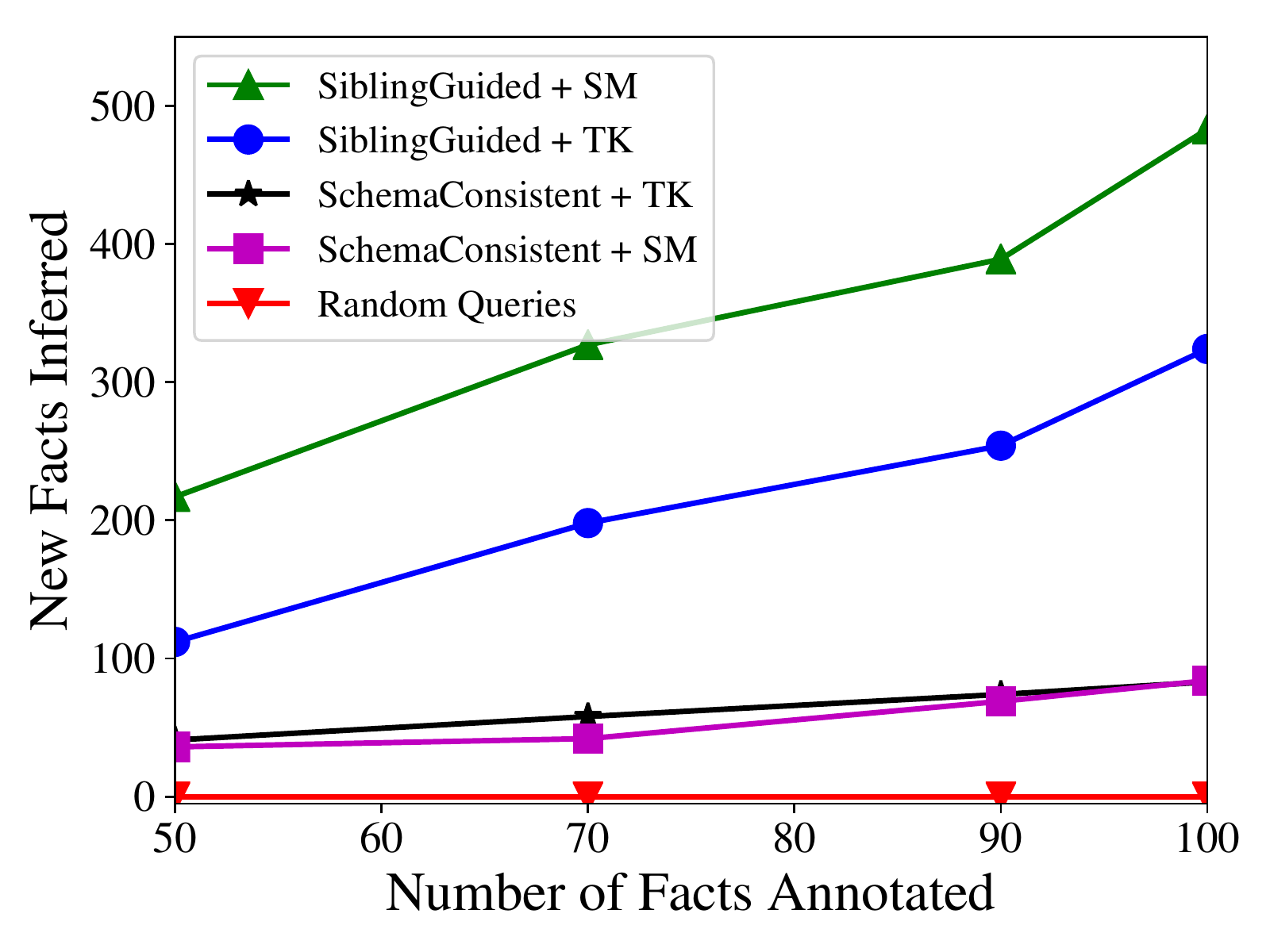}
\caption{\label{fig:subset_selection} Active Learning for new entities: Total number of new inferred facts (y-axis) for various human annotation query sizes (x-axis). The use of subset selection (green triangles, top) and sibling information (blue circles, 2$^\text{nd}$ from top) vastly outperforms various baselines.}
\end{figure}

In Table~\ref{table:AL-summary}, we assess the quality of $\widehat{L}$ in two ways, when $|\widehat{L}| = 100$: how many true facts does $\widehat{L}$ have and how many overall new facts does this annotation produce about $\tl{e}$. Figure~\ref{fig:subset_selection} provides a complementary view, focusing on the overall number of new facts inferred as $|\widehat{L}|$ increases. While these illustrative numbers are for a representative new entity, \textit{reindeer},
the overall trend and order of numbers remained the same for other new entities we experimented with.

We mention some highlights from Table~\ref{table:AL-summary}. First, not surprisingly, randomly choosing triples about $\tl{e}$ to annotate is ineffective. Second, choosing schema consistent triples results in 73 true triples (out of 100) but these facts help tensor factorization very little, resulting in only 10 additional new triples about $\tl{e}$. Our proposed sibling guided querying mechanism results not only in nearly all 100 facts being true along with 17 true facts inferred from sibling agreement (set M in Alg.~\ref{algo:AL}), but also, combined with submodular subset selection for balancing diversity with coverage (Alg.~\ref{algo:subset_selection}), ultimately results in 483 new facts about $\tl{e}$. These facts cover interesting new information such as \emph{(reindeer, eat, fruit)}, \emph{(wolf, chase, reindeer)}, and \emph{(reindeer, provide, fur)}.

Finally, the plot in Figure~\ref{fig:subset_selection} demonstrates that the qualitative trends remain the same, irrespective of the number $|\widehat{L}|$ of queries annotated. Overall, our sibling guided queries with submodular subset selection (green triangles, top-most curve) ultimately results in 5.8 times more new facts about $\tl{e}$ than a non-trivial, uncertainly based, schema consistent baseline (black stars, 3$^\text{rd}$ curve from the top). This attests to the efficacy of the method on this challenging problem and dataset.

\section{Conclusion}

This work explores KB completion for a new class of problems, namely completing generics KBs, which is an essential step for including general world knowledge in intelligent machines. The differences between generics and much studied named entity KBs make existing techniques either not scale well or produce facts at an undesirably low precision out of the box. We demonstrate that incorporating entity taxonomy and relation schema appropriately can be highly effective for generics KBs. Further, to address scarcity of facts about certain entities in such KBs, we present a novel active learning approach using sibling guided uncertainty estimation along with submodular subset selection. The proposed techniques substantially outperform various baselines, setting a new state of the art for this challenging class of completion problems.

Our method is applicable to KBs that have an associated entity taxonomy and relation schema. It is expected to be successful when information from siblings can be used to guide what is likely to be true and what is a good candidate to query for a given entity. We focus on KBs of generics where such information is available and---as we show---is highly valuable for effective KB completion.

Why does our use of types work substantially better in our setting than the use of types in various baselines? One hypothesis is the following. The use of complicated models requires substantial data and information. In our KB, the information appears so sparse and incomplete that using types in complicated ways is not productive. Our proposal instead attempts to use type information only to gently enhance the signal and reduce noise, before performing tensor decomposition. We hope this work will trigger further exploration of knowledge bases with generics, a key aspect of machine intelligence.




\subsection*{Acknowledgments}

\begin{small}
The authors would like to thank Peter Clark for fruitful discussions, valuable feedback, and crowdsourcing annotations; Matt Gardner for constructive comments and assessing graph-based completion methods on our datasets; and Udai Saini and Partha Talukdar for evaluating their CNTF approach on our datasets.
\end{small}

\ifthenelse{\isundefined{\forarxiv}}{
 \bibliographystyle{acl2012}
}{
 \bibliographystyle{plainnat}
}

\begin{small}
\bibliography{bib}
\end{small}

\end{document}